\documentclass[sigconf,screen=true,anonymous=false]{acmart}
\usepackage{graphicx}
\usepackage{amsmath}
\usepackage{amssymb}
\usepackage{mathtools}
\usepackage{comment}
\usepackage[subrefformat=parens,labelformat=parens]{subfig}
\usepackage{bm}
\usepackage{multirow}
\usepackage{threeparttable,booktabs}
\usepackage{tikz}
\usepackage{balance}
\usepackage{courier}                                       
\usepackage{cleveref}                                      
\usepackage[mathcal]{eucal}
\usepackage[]{algpseudocode}                               
\algrenewcommand\textproc{\texttt}
\makeatletter\let\float@addtolists\relax\makeatother
\usepackage{algorithm}
\usepackage{filecontents}                                  
\usepackage{pgfplots}
\usepackage{pgfplotstable}
\pgfplotsset{compat=newest}

\newcommand{\m}[1]{\mathbf{#1}}

\theoremstyle{plain}
\newtheorem{mytheorem}{\textbf{Theorem}}

\theoremstyle{definition}
\newtheorem{mydefinition}{\textbf{Definition}}
\newtheorem{myproblem}{\textbf{Problem}}

\usepackage[skip=1pt]{caption}            
\graphicspath{{./figs/}}

\setcopyright{none}

\definecolor{myorange}{RGB}{244,106,18} 
\definecolor{myblue}{RGB}{0,111,190}    
\definecolor{mygreen}{RGB}{0,127,128}   
\definecolor{myred}{RGB}{228,46,36}     
\definecolor{myyellow}{RGB}{198,148,34} 
\definecolor{mydark}{RGB}{114,44,114}   
\definecolor{mymiddle}{RGB}{144,44,144} 
\definecolor{mylight}{RGB}{167,44,167}  
\definecolor{item1}{RGB}{255,171,164}  
\definecolor{item2}{RGB}{142,202,206}
\definecolor{item2p}{RGB}{66,101,175}
\begin{document}

\title[]{
Bridging the Gap Between Layout Pattern Sampling and Hotspot Detection via Batch Active Learning
}

\iftrue
\author{Haoyu Yang}
\affiliation{
    \institution{CSE Department, CUHK}
}
\email{hyyang@cse.cuhk.edu.hk}

\author{Shuhe Li}
\affiliation{
    \institution{CSE Department, CUHK}
}
\email{shli@cse.cuhk.edu.hk}

\author{Cyrus Tabery}
\affiliation{
    \institution{ASML Brion Inc.}
}
\email{cyrus.tabery@asml.com}

\author{Bingqing Lin}
\affiliation{
    \institution{Shenzhen University}
}
\email{bqlin@szu.edu.cn}

\author{Bei Yu}
\affiliation{
    \institution{CSE Department, CUHK}
}
\email{byu@cse.cuhk.edu.hk}

\renewcommand{\shortauthors}{H.~Yang et al.}
\fi

\begin{abstract}
Layout hotpot detection is one of the main steps in modern VLSI design.
A typical hotspot detection flow is extremely time consuming due to the computationally expensive mask optimization and lithographic simulation.
Recent researches try to facilitate the procedure with a reduced flow including feature extraction, training set generation and hotspot detection,
where feature extraction methods and hotspot detection engines are deeply studied.
However, the performance of hotspot detectors relies highly on the quality of reference layout libraries which are costly to obtain and 
usually predetermined or randomly sampled in previous works.
In this paper, we propose an active learning-based layout pattern sampling and hotspot detection flow,
which simultaneously optimizes the machine learning model and the training set that aims to achieve similar or better hotspot detection performance with much smaller number of training instances.
Experimental results show that our proposed method can significantly reduce lithography simulation overhead while attaining satisfactory detection accuracy on designs under both DUV and EUV lithography technologies.
\end{abstract}

\maketitle

\section{introduction}
\label{sec:intro}
Along with aggressive feature size scaling, even equipped with various resolution enhancement techniques and hierarchical design strategy, modern chip designs are more and more complicated and greatly challenged by manufacturability issues.
VLSI layout hotspot detection is one of the most critical steps in manufacturability-aware design, which is costly to estimate because of the complicated mask optimization and lithography simulation.
Many researches have been conducted to facilitate the procedure which usually share a flow as shown in \Cref{fig:train-set}, 
including \textit{feature extraction}, \textit{training set generation} and \textit{hotspot detection}.

Feature extraction aims to convert layout geometry information (e.g.~density \cite{HSD_SPIE2015_Matsunawa,HSD_TCAD2014_Wen}, frequency \cite{HSD_DAC2013_Zhang,HSD_DAC2017_Yang} and design rule \cite{HSD_TCAD2015_Yu}) into reduced mathematical representations,
which are expected to improve hotspot detection accuracy.
Recently, deep neural networks also exhibit powerful feature learning ability that can obtain layout representations without prior knowledge \cite{HSD_DAC2017_Yang,HSD_SPIE2017_Yang,HSD_SPIE2016_Matsunawa,HSD_JM3_2016_Shin}.
In the hotspot detection stage, all selected samples and labels are fed into hotspot detection engines based on, mostly, pattern matching and machine learning.
In a pattern matching flow, similar patterns within a specific radius are clustered together based on the constraints of translation, area and/or edge displacements \cite{HSD_TCAD2014_Wen,HSD_DAC2017_Chang,HSD_DAC2017_Chen}.
Lithography simulation will be performed on the representative clips results from which will be then labeled to the whole cluster.
Above process indicates that fuzzy matching results are drastically affected by in-cluster variance.
Although aggressive constraints can be introduced to ensure a low in-cluster variance, additional cluster count will significantly increase lithography simulation overhead.
On the other hand, machine learning technologies tackle the problem through fitting layout representations into efficient machine learning models. 
\cite{HSD_TCAD2015_Yu,HSD_ASPDAC2012_Ding} employ support vector machine for efficient hotspot detection.
\cite{HSD_SPIE2015_Matsunawa,HSD_ICCAD2016_Zhang} enhance hotspot detectors with boosting algorithms and additional learning strategies.
\cite{HSD_DAC2017_Yang,HSD_SPIE2017_Yang,HSD_SPIE2016_Matsunawa,HSD_JM3_2016_Shin} adopt emerging deep neural networks that automatically learn layout representations and perform classification.
Still, overfitting problem is inevitable due to weakly distributed training data.

Previous works show that although pattern matching-based methods and machine learning-based methods exhibit different functionalities, they all rely highly on the quality of reference layout libraries.
For example, in-cluster pattern variance directly affects pattern matching results and pattern diversity contributes to the generality of trained machine learning models.
Layout pattern sampling problems are addressed by several works that are, to some extent, related to clustering approaches. 
Representative methods include clustering on frequency domain \cite{HSD_DAC2013_Zhang,HSD_JM3_2015_Shim}, Bayesian clustering \cite{HSD_JM3_2016_Matsunawa}, and clustering based on layout topology \cite{HSD_DAC2017_Chang,HSD_DAC2017_Chen,HSD_JM3_2015_Shim,HSD_DAC2012_Guo}. 
However, sampling and hotspot detection are mostly conducted exclusively which ignores the beneath integrity between them.

\begin{figure*}[tb!]
	\centering
	\includegraphics[width=.88\linewidth]{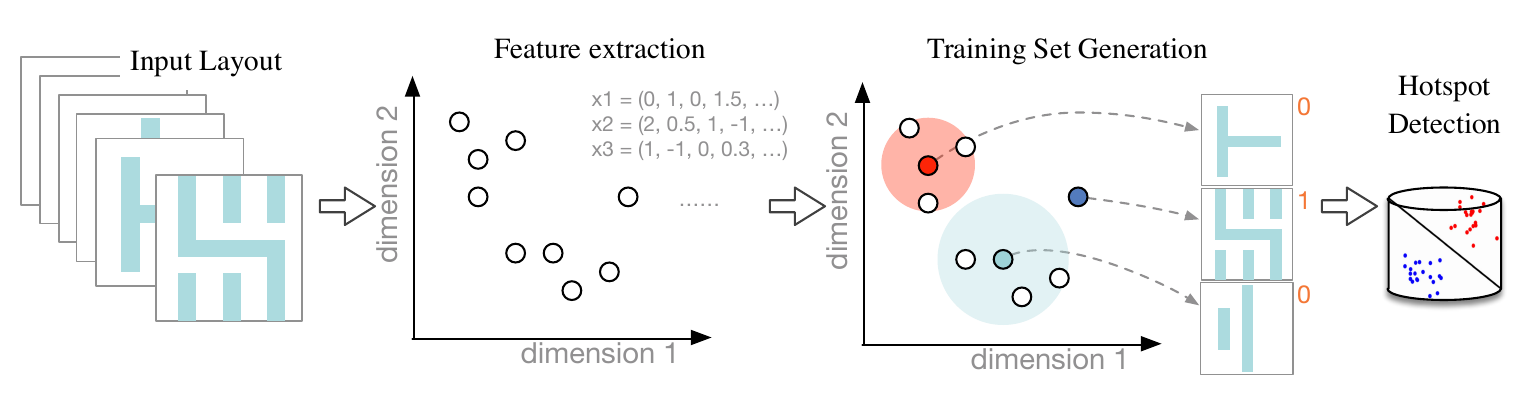}
	\caption{A conventional process of layout hotspot training set and detection model generation.}
	\label{fig:train-set}
\end{figure*}

In this paper, we will propose an active learning-based framework that can bridge the gap between the layout pattern sampling procedure and the hotspot detection problem.
Active learning targets at machine learning problems with massive data that is costly and time consuming to label.
A major step of active learning is querying instances to determine whether the instance should be labeled and added into the training set from a perspective of machine learning model generality \cite{ATL_TR2010_Settles}.
Representative querying strategies include uncertainty sampling (US \cite{ATL_SIGIR1994_Lewis}), query by committee (QBC \cite{ATL_ML1997_Freund}),
and expected model change (EMC \cite{ATL_ICDM2013_Cai}).
US aims to find the instances which the prediction model is most uncertain about and have the posterior probability around 0.5, QBC selects instances based on the disagreement among multiple classifiers and EMC labels most influential data in terms of the existing model.
A common idea behind these strategies is labeling instances that are hardly distinguished by the classifier.
However, there are several drawbacks of existing active learning strategies:
(1) only one sample is selected in each iteration in most active learning flows which is lacking in efficiency;
(2) machine learning models have to be retrained from raw state once the training set is updated;
(3) training set diversity is not considered in sampling flow which might cause serious overfitting problem \cite{ATL_TR2010_Settles,ATL_JMLR2001_Tong,ATL_ICML2000_Greg}.
Although K-L divergence on posterior probabilities of unlabeled samples can be applied for diversity analysis \cite{ATL_TPAMI2015_Shayok}, the effectiveness is limited on binary classification problems.

To address these concerns, we propose a batch mode active learning method that considers both model uncertainty and training set diversity.
We embed the active learning engine into deep neural networks thus data sampling and incremental model training can be conducted alternatively.
Guaranteed by the on-line property of stochastic gradient descent, we only need to finetune the neuron weights according to new labeled instances instead of training model from scratch in each iteration. 
The rapid development of deep neural networks makes it possible to learn representative features from raw image.
We take advantage of this characteristic and construct a diversity matrix of automatically learned features which will contribute as a partial criterion for data sampling in each iteration.
The main contributions of this paper are listed as follows:
\begin{itemize}
	\item A novel layout pattern sampling and hotspot detection flow is proposed to simultaneously optimize training set and machine learning model.
	\item We develop a batch mode active learning engine that samples multiple instances in each iteration according to the training set diversity and data uncertainty, 
	where a specific distance metric is designed to guarantee a convex objective that makes the sampling procedure more efficient.
	\item We conduct experiments on metal layers under 7$nm$ and 28$nm$ technology nodes which demonstrate the generality of the proposed flow that significantly increases hotspot detection accuracy while minimizing lithography simulation overhead.
\end{itemize}

The rest of this paper is organized as follows.
\Cref{sec:prelim} introduces basic terminologies and definition.
\Cref{sec:algo} lists theoretical and algorithmic details.
\Cref{sec:exp} presents experiment settings and results, followed by conclusion in \Cref{sec:conclu}.

\section{Preliminaries}
\label{sec:prelim}
This section introduces some terminologies and related problem formulation.
Throughout this paper, scalers are written as lowercase letters (e.g.~$x$), vectors are bold lowercase letters (e.g.~$\m{x}$) and matrices are represented as bold uppercase letters (e.g.~$\m{X}$).
Particularly, we use $J_n(\cdot)$ to represent the Bessel function of the first kind of order $n$.
The framework evaluation metrics are defined as follows.
 
\begin{mydefinition}[Hit]
	A hit is defined as when the detector reports hotspot on a clip of which at least one defect occurs at the core region. 
	We also denote the ratio between number of hits and total hotpsot clips as detection accuracy. 
\end{mydefinition}
\begin{mydefinition}[Extra]
	An extra is defined as when the detector reports hotspot on a clip of which no defect occurs at the core region. 
\end{mydefinition}
\begin{mydefinition}[Litho-clip]
	A litho-clip is a clip in the training set or an extra that is labeled hotspot or non-hotspot based on results of lithography simulation.
	The count of litho-clips reflects the lithography simulation overhead.
\end{mydefinition}
According to the evaluation metrics above, we define the problem of layout pattern sampling and hotspot detection (PSHD) as follows.
\begin{myproblem}[PSHD]
	Given a layout design, the objective of PSHD is sampling representative clips that will generalize the hotspot pattern space and maximize the machine learning model generality, i.e.~,
	maximizing the detection accuracy while minimizing the number of litho-clips.
\end{myproblem}

\section{The Algorithm}
\label{sec:algo}

In this section, we will discuss the details of our pattern sampling and hotspot detection flow, including the lithography proximity effect, the batch active sampling algorithm and some analysis.

\begin{figure}[tb!]
	\centering
	\includegraphics[width=.88\linewidth]{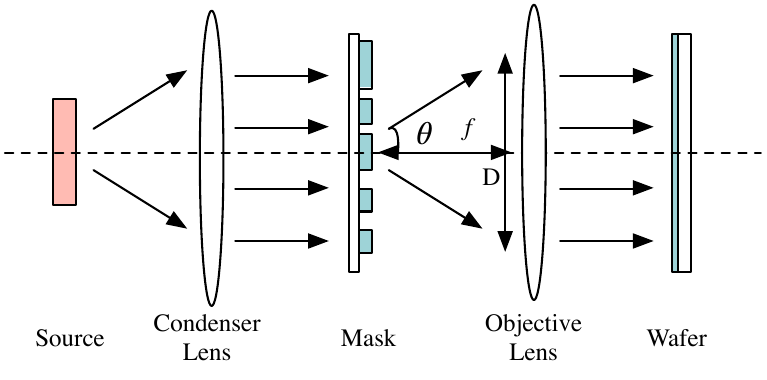}
	\caption{A simple lithography imaging system.}
	\label{fig:imaging-sys}
\end{figure}
\subsection{Lithography Proximity Effect}
\label{sec:lpe}
Most challenges in optical lithography come from the proximity effects caused by diffraction as light goes through the mask stage, as shown in \Cref{fig:imaging-sys}.
In a lithography imaging system, the electric field of the diffracted pattern is given by the Fourier transform of the original mask pattern.
Afterwards, diffracted patterns will be collected by the objective lens to project images on the wafer.
Because of the limited size of objective lens, higher order diffraction patterns will be discarded when forming the image on the wafer that results in a lower pattern fidelity \cite{DFM_B2008_Mack}.
Typically, to ensure the mask image can be transferred onto the wafer as accurate as possible, at least the zero and $\pm1$st diffraction order should be captured by the objective lens.
Accordingly, the smallest design pitch can be defined as \Cref{eq:resolution},
\begin{align}
	\label{eq:resolution}
	\dfrac{1}{p}=\dfrac{NA}{\lambda},
\end{align}
where $p$ denotes design pitch, $\lambda$ is the wavelength of the light source and $NA$ is the numerical aperture of the objective lens
which determines how much information can be collected by the objective lens and is given by
\begin{align}
	\label{eq:na}
	NA=n \sin \theta_{\max}=\dfrac{D}{2f},
\end{align}
where $n$ is the index of refraction of the medium, $\theta_{\max}$ is the largest half-angle of the diffraction light that can be collected by the objective lens,
$D$ denotes the diameter of physical aperture seen in front of the objective lens and $f$ represents the focal length \cite{DFM_B2004_John}.

The existence of diffraction makes it also interesting to analyze the minimum distance when two shapes stop affecting the aerial images of each other.
Fraunhofer diffraction occurs in classic lithography system, where the diffracted patten is determined by the Fraunhofer diffraction integral \cite{DFM_B2004_John}.
For simplicity, we consider the contact hole as an example whose diffraction pattern resembles the Airy disk.
The light intensity in terms of observation angle $\theta$ at the entrance of the objective lens is shown in \Cref{eq:airy}.
\begin{align}
	\label{eq:airy}
	I(\theta)= (\dfrac{2J_1 (kr \sin \theta)}{kr \sin \theta})^2 \cdot I_0,
\end{align}  
where $I_0$ denotes center intensity of airy disk, $r=\frac{D}{2}$ is the radius of the entrance pupil and $k=\frac{2 \pi}{\lambda}$ is the wavenumber.
According to the properties of Bessel function, dark regions of Airy disk that correspond to zeros of $I(\theta)$ appear periodically with a degradation of total energy.
The total energy within an observation angle can be derived by integrating \Cref{eq:airy} over $\theta$,
\begin{align}
	P(\theta)=[1-J_0^2(kr \sin \theta)-J_1^2(kr \sin \theta)],
\end{align}
which reflects by how much the diffraction information can be collected.
If we pick the 6$th$ zero point of $I(\theta)$ at $kr\sin \theta\approx19$, we can derive $P(\theta)\approx96.73\%$, which is the fraction of diffraction collected with in a given window size.
Besides, we assume $n=1$ in the air,
\begin{align}
	\label{eq:maxdis}
	\sin \theta = \dfrac{19}{kr}=NA.
\end{align}
Combine \Cref{eq:maxdis} and \Cref{eq:na},
\begin{align}
	\label{eq:isodis}
	D=6.05 \cdot \dfrac{\lambda}{NA}.
\end{align}
Here $D$ determines the minimum distance when two shapes can be regarded as isolated patterns,
which can be derived to be 230$nm$ using $NA=0.35$ and $\lambda=13.5nm$ under extreme ultraviolet (EUV) lithography technologies.

\subsection{Diversity Aware Batch Sampling}
\label{sec:sample}

\begin{figure*}[t!]
	\centering
	\subfloat[]{\includegraphics[width=.274\textwidth]{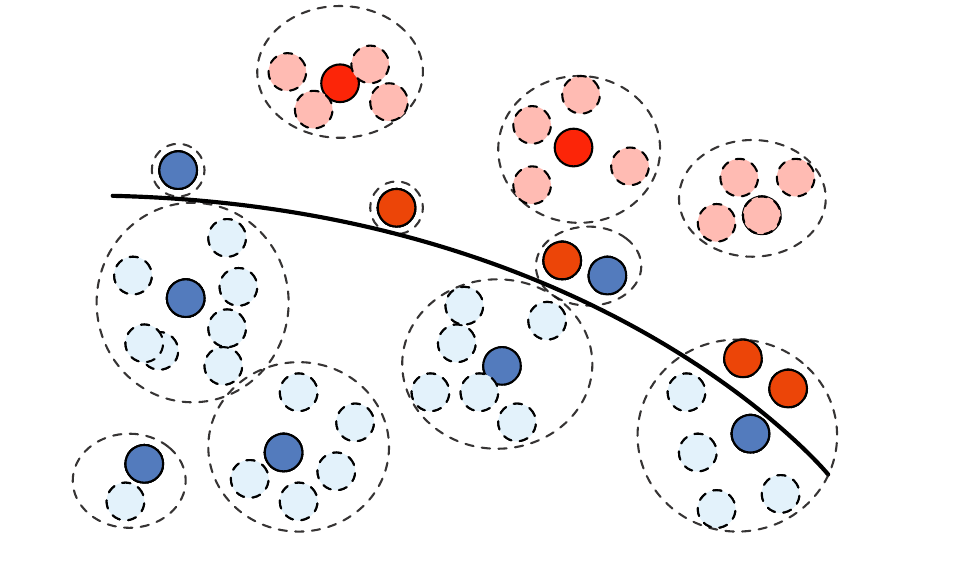} \label{fig:sample-pm}}
	\subfloat[]{\includegraphics[width=.274\textwidth]{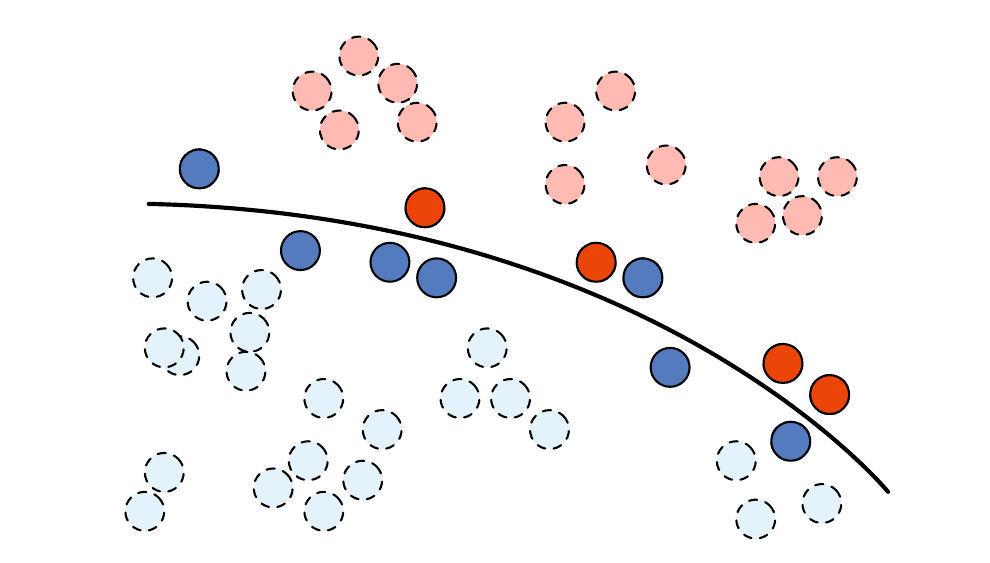} \label{fig:sample-al}}
	\subfloat[]{\includegraphics[width=.274\textwidth]{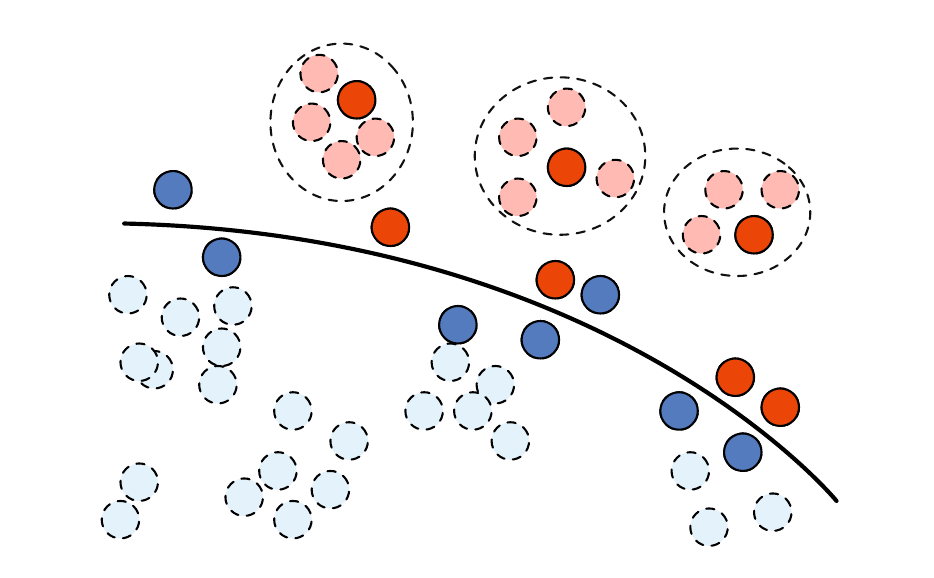}\label{fig:sample-hsd}}
    \includegraphics[width=.18\textwidth]{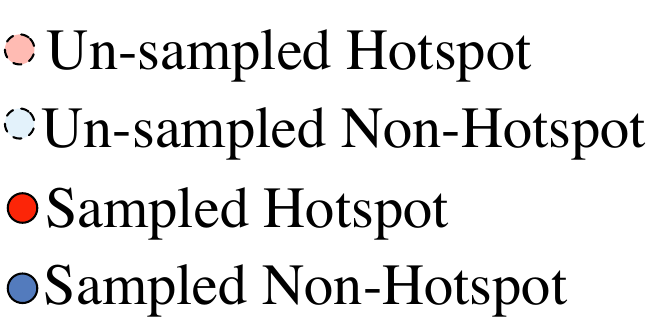}
	\caption{Visualization of different layout pattern sampling methods: (a) Pattern matching; (b) Conventional active learning; (c) Proposed pattern sampling and hotspot detection flow.}
	\label{fig:sample}
\end{figure*}

Because it is extremely costly to label layout clips, 
our flow aims to sample as little number of clips as possible while ensuring good machine learning model generality.
Conventional layout pattern sampling methods conduct clustering on layout clips and obtain representative patterns based on the results of pattern matching or clustering.
Although the clustering can effectively reduce the sample number, it does not consider the behavior or requirement of, especially, machine learning-based hotspot detectors.
As shown in \Cref{fig:sample-pm}, pattern matching collects a lot of less critical patterns that lie far from the decision boundary while ignoring important patterns.
In conventional active learning-based sampling (see \Cref{fig:sample-al}), prediction uncertainty of each clip is included in the selection criteria.
That is, patterns with posterior probability around 0.5 will be sampled with higher priority.
However, in a layout pattern sampling and hotspot detection task, we care more about hotspot regions.
Therefore, apart from considering diversity of training instances, we tend to select clips with higher probability being hotspot in our sampling approach, as illustrated in \Cref{fig:sample-hsd}.

Discriminative machine learning models are usually designed to find the optimal hyperplane that separates the whole data space. 
The quality of a model is measured by its generative loss which is associated with the prediction error on the future instances.
In this section, we will discuss an instance selection policy considering both \textit{model uncertainty} and  \textit{data diversity},
thus the selected instances are expected to contribute most on the trained model generality.
The uncertainty describes how confident of the classifier when recognizing new instances.
A model is uncertain on a given instance if the prediction probability draws around 0.5 according to the posterior distribution or the instance is too close to the hypothesis plane in the feature space.
Data diversity corresponds to the instance distribution in the data set.
The underneath idea is to label instances into training set such that the training set entropy is maximized.
Most active learning algorithms, such as US, QBC and EMC, are designed to pick one instance in each iteration, which is not efficient as problem sizes grow.
Even these methods are applied for batch selection, samples touch the selection criteria are labeled into train set,
when redundant instance samples are more likely to be chosen.
Here we consider a \textbf{batch selection} mechanism that takes both model uncertainty and data diversity into account.

Given a training set $\mathcal{L}_t$ and an unlabeled set $\mathcal{U}_t$ at time $t$.
Let $\mathbf{w}_t$ be the classifier parameters trained on $\mathcal{L}_t$. 
The objective is to select a batch $\mathcal{B}$ with $k$ points from $\mathcal{U}_t$ so that the future learner $\mathbf{w}_{k+1}$,
trained on $\mathcal{L}_t \cup \mathcal{B}$, has maximum generalization capability.
Let $\mathcal{Y}=\{0,1\}$ be the set of possible classes in the problem.
For a given unlabeled layout clip $\mathbf{x}_i$, we denote the related posterior probability as $p(y|\mathbf{x}_i;\mathbf{w}_t)$.
Usually, the uncertainty of the unlabeled instance $\mathbf{x}_i$ is defined as the entropy of the predicted probabilities, as shown in \Cref{eq:entrp}.
\begin{align}
	c(i)= -\sum_{j \in \mathcal{Y}} p(y=j|\m{x}_i;\m{w}_t) \log p(y=j|\m{x}_i;\m{w}_t).
	\label{eq:entrp}
\end{align}
However, in the layout pattern sampling problems, problematic instances are of more interests.
We therefore pick a simple but more practical representation of $c(i)$,
\begin{align}
		c(i)=  p(y=1|\m{x}_i;\m{w}_t),
	\label{eq:ucthsd}
\end{align}
which corresponds to the probability of a given instance being hotspot.
Usually, the redundancy between unlabeled points $\mathbf{x}_i$ and $\mathbf{x}_j$ can be calculated through K-L divergence, 
which measures how two training instances differ from each other in a statistic point of view.
In the domain of layout hotspot detection, however, we are dealing with yes or no problem, which is less informative for diversity analysis. 
To benefit the layout analysis problem, we use inner-product of two instances in the normalized feature space as shown in \Cref{eq:kldf}.
\begin{align}
    E(i,j)=\m{x}_i^\top\m{x}_j. \label{eq:kldf}
\end{align}
We can further formulate the diversity matrix $\m{D} \in \mathbb{R}^{n\times n}$, whose entries are defined by \Cref{eq:kldf}.
 
Given the matrix $\m{D}$, the batch mode active learning problem is shown in mathematical formulation \eqref{eq:bmal},
where the objective is to select a batch of points with high aggregate uncertainty scores and high divergences among the samples.
\begin{subequations}
    \label{eq:bmal}	
    \begin{align}
        \min_{\m{m}} ~~~& \m{m}^\top\m{D} \m{m},      \tag{\ref*{eq:bmal}}\\
        \textrm{s.t}.~~ & m_i \in \{0, 1\},\forall i, \\
                        & \sum_{i} m_i = k,\forall i. 
    \end{align}
\end{subequations}
Here $k$ is the number of patterns that will be selected into the training set, $m_i$ is a binary variable, and $m_i = 1$ if pattern $\mathbf{x}_i$ is selected in the batch $\mathcal{B}$.
It should be noted that Formula \eqref{eq:bmal} is binary quadratic programming, which is NP-hard.
We relax the integer constrains and derive the following problem, 
\begin{subequations}
    \label{eq:bmalr}	
    \begin{align}
        \min_{\m{m}} ~~~& \m{m}^\top\m{D} \m{m},    \tag{\ref*{eq:bmalr}}\\
        \textrm{s.t}.~~ & m_i \in [0, 1],   \forall i, \\
                        & \sum_{i} m_i = k, \forall i, 
    \end{align}
\end{subequations}
which is a standard quadratic programming problem and can be solved efficiently.
It can be seen here one advantage of the proposed distance metric over KL-divergence and Euclidean distance is that \Cref{eq:kldf}
ensures the objectives of \eqref{eq:bmal} and \eqref{eq:bmalr} to be convex by $\m{D} \succeq 0$.
Finally, the integer solution can be recovered by picking $k$ largest entries in $\m{m}$.

The rapid development of deep neural networks makes it possible to learn representative features from raw image and complete effective classification jobs.
Therefore, in this project, we pick up a shallow convolutional neural networks as the preferred machine learning model which will be embedded into the active learning flow.
In particular, features obtained from the fully-connected layers are fed into \Cref{eq:kldf} to calculate the divergence matrix.
Most neural networks are trained with mini-batch gradient descent (MGD), where a random small batch of training samples are fed into the neural networks to update neuron weights.
The online property of MGD makes it easier to update the model on new instances without retraining the model from scratch compared to traditional support vector machine or logistic regression.
It should be noted that the proposed classification driven active learning flow is very general that it can be plugged into any incremental hotspot detectors.

In most cases $\m{D}$ will be extremely large, especially for EUV specific layers, which makes Formula~\eqref{eq:bmalr} hard to solve.
We therefore stochastically sample a subset $\hat{\mathcal{U}}_t \subseteq \mathcal{U}_t$  before entering the quadratic programming phase to further reduce the computational cost.
Finally, the neural network can be accordingly updated as $\mathbf{w}_{t+1} = \mathbf{w}_{t} + \alpha \cfrac{\partial l}{\partial \m{w}_t}$.
Here $\alpha$ denotes the updating rate and $l$ is the average cross-entropy loss of sampled instances, defined as follows:
\begin{align}
	l=\frac{1}{k} \sum_{i=1}^{k} \log p(y_i=1|\m{x}_i; \mathbf{w}_t).
\end{align}

It should be noted that although the neural networks may need multiple iterations to finish training, the computational cost is much less than training from a raw model.
\cite{HSD_DAC2017_Yang} has shown that biased label is able to provide better trade-offs on hotspot detection problem during the fine-tune procedure.
However, by our observation, stepped bias significantly disturbs the pre-trained model.
We therefore improves this technique by letting the bias change linearly along with the training step.

\begin{algorithm}[h]
	\caption{Batch Active Sampling}
	\label{alg:bas}
	\begin{algorithmic}[1]
		\Require $\mathcal{L}_0,\mathcal{U}_0,n,\sigma$.
		\Ensure $\m{w},\mathcal{D}$.
		\State Initialize $\mathbf{w}\sim \mathcal{N}(0, \sigma)$;
		\State $\mathcal{L}\leftarrow \mathcal{L}_0,\mathcal{U} \leftarrow \mathcal{U}_0,\mathcal{D}\leftarrow \emptyset$;
		\State $\m{w}\leftarrow$Train the machine learning model based on $\mathcal{L}$.
		\While{$\mathcal{U} \neq \emptyset$}
		\State  $\hat{\mathcal{U}}\leftarrow$ Sample $n$ instances with highest probability (predicted with current $\mathbf{w}$) being hotspot from $\mathcal{U}$;
		\State $\mathcal{U}\leftarrow \mathcal{U} \backslash \hat{\mathcal{U}}$;        
		\State $\mathcal{B}\leftarrow$Select $k$ instances by solving problem~\eqref{eq:bmalr};
		\State $\mathcal{L}\leftarrow \mathcal{L}\cup\mathcal{B}$;
		\State $\mathcal{D}\leftarrow\hat{\mathcal{U}}\backslash \mathcal{B}\cup \mathcal{D}$;
		\State $\mathbf{w}\leftarrow$Update machine learning model based on $\mathcal{L}$;
		\EndWhile
		\State \Return $\m{w},\mathcal{D}$.
	\end{algorithmic}
\end{algorithm}

\Cref{alg:bas} presents the details of the layout pattern sampling flow.
The algorithm requires an initial training set $\mathcal{L}=\mathcal{L}_0$ with labeled patterns, an unlabeled pattern pool $\mathcal{U}=\mathcal{U}_0$, 
number of patterns to be queried $n$ and a standard deviation $\sigma$ used to initialize the machine learning models (lines 1--2);
we first train an initial machine learning model based on $\mathcal{L}_0$ (line 3).
In each sampling iteration, we fetch $n$ instances from $\mathcal{U}$ without replacement and form a query set $\hat{\mathcal{U}}$ (lines 5--6);
$k$ instances are sampled into a set $\mathcal{B}$ by solving problem~\eqref{eq:bmalr} (line 7);
new training set $\mathcal{L}$, discarded set $\mathcal{D}$ and the machine learning model are updated accordingly (lines 7--9).
The algorithm ends when the unlabeled instance pool is empty and returns the trained model and remaining unlabeled patterns to be verified by the machine learning model.

Note that \Cref{alg:bas} requires an initial labeled dataset $\mathcal{L}_0$ to obtain a pre-trained model that will be used to extract features for future layout patterns.
Thus, $\mathcal{L}_0$ is critical on the performance of the whole flow.
Because it is almost impossible to know which pattern is more likely to have defects at beginning, we only consider the diversity of layout features.
Because learned features are not available without a trained CNN model, we design alternate features considering the lithography process.
In \Cref{sec:prelim}, we have shown that frequency components of layout patterns contributes most to layout printabilities,
we therefore perform feature tensor extraction \cite{HSD_DAC2017_Yang} on each layout clip.
To make the computation efficient, we pick only the second channel of the feature tensor as our feature vector to calculate the diversity matrix,
which corresponds to the frequency components that have largest information. 
The initial training set $\mathcal{L}_0$ can then be obtained by solving problem~\eqref{eq:bmalr} with updated $\mathbf{D}$.


\subsection{Algorithm Analysis}
\label{sec:analysis}
In this section, we will discuss and analysis some technique details of our proposed framework. 
As described in previous section, we relax the integer constraints when solving the sampling problem~\Cref{eq:bmal}.
Because each queried instance will be sampled or dropped by solving problem~\Cref{eq:bmalr}, the entries of the optimal solution will be rounded into binary values.
Here we will analysis the loss of optimality of problem~\Cref{eq:bmalr} when reconstructing an integer solution as the sampling choice, as claimed in Theorem~\ref{thm:maxgap}.
\begin{mytheorem}
	\label{thm:maxgap}
	Let $\mathbf{m}$ be the optimal solution of problem~\eqref{eq:bmalr} that is binarized into $\mathbf{m}_b$ by setting $k$ largest entries to 1 and rest $n-k$ entries to 0,
	then 
	\begin{align}
		f(\mathbf{m}) \le f(\mathbf{m}_b) \le 2f(\mathbf{m})+2\lambda_n (k-\frac{k^2}{n}),
	\end{align}
	where $f(\mathbf{x})=\mathbf{x}^\top\mathbf{D}\mathbf{x}$, $n$ is total number of instances in each query iteration, $k$ is the number of instances that will be sampled into training set and $\lambda_1 \le \lambda_2 \le \dots \le \lambda_n$ are the eigenvalues of $\mathbf{D}$. 
\end{mytheorem}
\begin{proof}
	$f(\mathbf{m}) \le f(\mathbf{m}_b)$ is trivial, and we will show that $f(\mathbf{m}_b) \le 2f(\mathbf{m})+2\lambda_n (k-\frac{k^2}{n})$.
	According to Equation~\eqref{eq:kldf}, the distance matrix $\mathbf{D}$ can be written as $\mathbf{D}=\mathbf{F}^\top \mathbf{F}$,
	where each column of $\mathbf{F}$ is the feature vector of each queried instance, thus $\mathbf{D} \succeq \mathbf{0}$ and $f$ is convex.
	By definition,
	\begin{align}
		\frac{1}{2}f(\mathbf{m})+\frac{1}{2}f(\mathbf{m}_b-\mathbf{m}) \ge f(\frac{1}{2}\mathbf{m}_b),
	\end{align}
	i.e.
	\begin{align}
		\frac{1}{2} \mathbf{m}_b^\top \mathbf{D} \mathbf{m}_b - \mathbf{m}^\top \mathbf{D} \mathbf{m} \le (\mathbf{m}_b-\mathbf{m})^\top \mathbf{D}(\mathbf{m}_b-\mathbf{m}).
	\end{align}
	By Rayleigh-Ritz theorem \cite{MTRX_B2012_Gene},
	\begin{align}
		(\mathbf{m}_b-\mathbf{m})^\top \mathbf{D}(\mathbf{m}_b-\mathbf{m}) \le \lambda_n ||\mathbf{m}_b-\mathbf{m}||_2^2.
		\label{eq:rrthm}
	\end{align}
	Claim that 
	\begin{align}
		\label{eq:minmaxieq}
		||\mathbf{m}_b-\mathbf{m}||_2^2 \le \max_{\mathbf{y}\in \mathcal{T}}||\mathbf{m}_b-\mathbf{y}||_2^2 = \max_{\mathbf{y} \in \mathcal{T}} \min_{\mathbf{x} \in \mathcal{T}_b,\mathbf{y} \in \mathcal{T}} ||\mathbf{x}-\mathbf{y}||_2^2,
	\end{align}
	where $\mathcal{T}=\{\mathbf{x}\in \mathbb{R}^n|\sum_{1}^{n}x_i=k, x_i\in[0,1],\forall i\}$ and $\mathcal{T}_b=\{\mathbf{x}\in \mathbb{R}^n|\sum_{1}^{n}x_i=k, x_i\in\{0,1\},\forall i\}$.
	Without loss of generality, we assume all the entries of a given $\mathbf{y}$ are placed in an order
	\begin{align}
		1 \ge y_{\delta_1} \ge y_{\delta_2} \ge \dots \ge y_{\delta_n} \ge 0,
	\end{align}
	thus according to the rounding strategy, $\mathbf{m}_b$ is defined as follows,
	\begin{align}
		m_{b,i}=
		\begin{cases}
			1, &\forall i \in \{\delta_1, \delta_2, \dots, \delta_k\},\\
			0, &\text{otherwise}.
		\end{cases}
	\end{align}
	Then,
	\begin{align}
		||\mathbf{m}_b-\mathbf{y}||_2^2&=\sum_{i=1}^{k} (1-y_{\delta_i})^2 + \sum_{i=k+1}^{n} y_{\delta_i}^2 \nonumber \\
		&=k+\sum_{i=1}^{n} y_{\delta_i}^2 - 2\sum_{i=1}^{k} y_{\delta_i} \nonumber \\
		&\le k+\sum_{i=1}^{n} y_{\delta_i}^2 - 2\sum_{i=1}^{k} y_{\eta_i} \nonumber \\
		&=||\mathbf{x}-\mathbf{y}||_2^2, \forall \mathbf{x} \in \mathcal{T}_b, \mathbf{y} \in \mathcal{T},
	\end{align}
	as claimed in Equation~\eqref{eq:minmaxieq}.\\
	Consider a right pyramids with a regular base, which has its apex at the origin, $C^k_n$ edges defined by the vectors defined in $\mathcal{T}_b$ and a regular polygon base $A$ lies in the hyperplane $\sum_{i=1}^{n} x_i = k$.
	As shown in Figure~\ref{fig:thm1}, $||\mathbf{m}_b - \mathbf{m}||_2^2$ reaches its maximum value when $\mathbf{m}$ lies in the center of the pyramid base.
	\begin{figure}[t!]
		\centering
        \subfloat[]{\includegraphics[width=.16\textwidth]{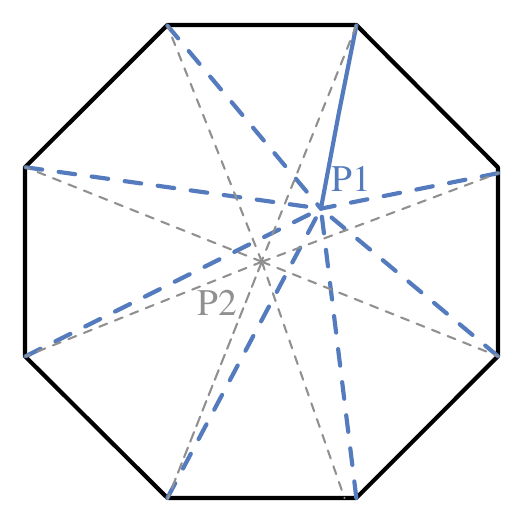} \label{fig:thm1-1}}  \hspace{.4in}
        \subfloat[]{\includegraphics[width=.16\textwidth]{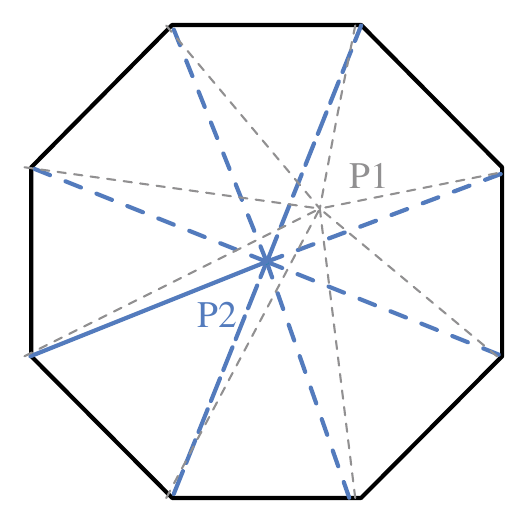} \label{fig:thm1-2}}
		\caption{Geometric view of $\mathbf{m}_b - \mathbf{m}$. $P_1$ and $P_2$ denote the end points of $\mathbf{m}$, and in particular, $P2$ lies in the center of the base polygon. The solid segments in each figure represents $||\mathbf{m}_b - \mathbf{m}||_2^2$ for a given $\mathbf{m}$.}
		\label{fig:thm1}
	\end{figure}
	That is	$\mathbf{m}=[\frac{k}{n},\frac{k}{n},\dots,\frac{k}{n}]^\top$, and 
	\begin{align}
		\max_{\mathbf{m}_b, \mathbf{m}} ||\mathbf{m}_b - \mathbf{m}||_2^2 = k (1-\frac{k}{n})^2 + (n-k) (\frac{k}{n})^2=k-\frac{k^2}{n},
	\end{align}
	which, combined with Equation~\eqref{eq:rrthm}, justifies the theorem. 
\end{proof}
Theorem~\ref{thm:maxgap} provides a theoretical guidance on choosing proper $n$ and $k$ in the batch sampling procedure, 
which can also be intuitively explained by the fact that if we sampling all or one instances in each querying iteration,
we have no risk on the integer relaxation error, however, at the cost of diversity loss.

\subsection{The Flow Summary}
The proposed layout pattern sampling and hotspot detection flow is illustrated in \Cref{fig:flow}.
To analysis the printability of a full chip design, we dispatch the layout into clips based on the lithography proximity effect analysis (\Cref{sec:lpe}), 
such that the whole chip is covered by the core region of each clip that contains enough information to conduct printability estimation.
And then, the training set and the machine learning model will be updated until convergence (\Cref{sec:sample}), when all the clips will be either labeled or dropped.
Finally, the full chip hotspot detection will be conducted on the dropped clips with the final learning model.
\begin{figure}[tb!]
	\centering
	\includegraphics[width=0.45\textwidth]{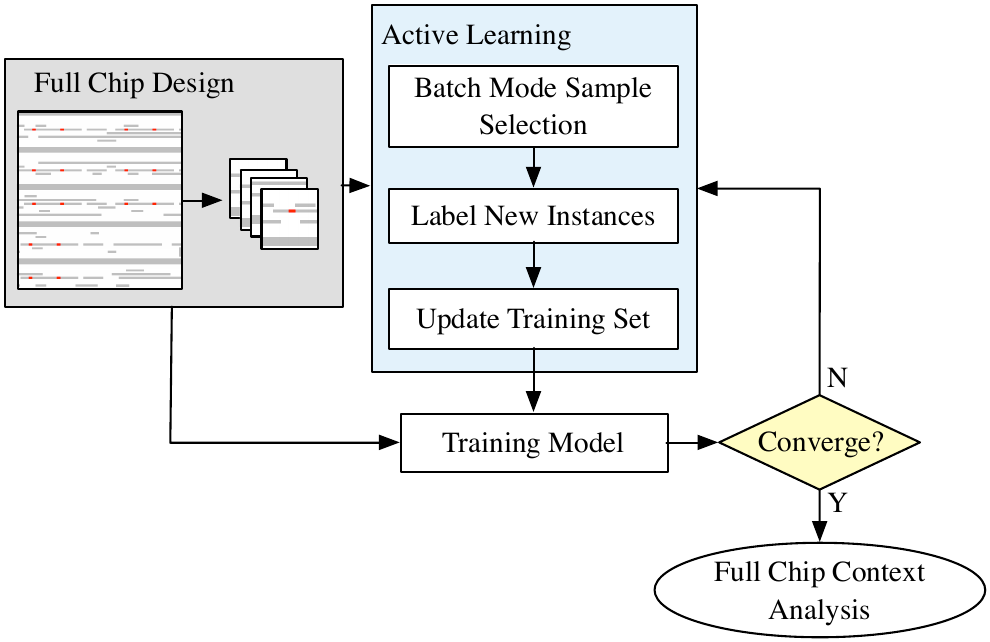}
	\caption{Pattern sampling and hotspot detection flow.}
	\label{fig:flow}
\end{figure}

\section{Experimental Results}
\label{sec:exp}
\subsection{Experimental Setup}
Our layout pattern sampling and hotspot detection flow is tested on ICCAD2012 (denoted as \texttt{ICCAD12}) \cite{HSD_ICCAD2012_Torres} and ICCAD2016 (denoted as \texttt{ICCAD16}) \cite{HSD_ICCAD2016_Rasit} CAD contest benchmark sets.

\begin{table}[b!]
	\centering
	\caption{Benchmark Details}
	\label{tab:bench}
	\renewcommand{\arraystretch}{1.1}
	\setlength{\tabcolsep}{1pt}
	\begin{tabular}{|c|c|c|c|c|}
		\hline
		Benchmarks  & CD ($nm$) & HS \# & NHS \# & Tech ($nm$) \\ \hline\hline
		\texttt{ICCAD12}   & 45           & 3728  & 159672 & 28        \\
		\texttt{ICCAD16-1} & 16             & 0     & 63     & 7         \\
		\texttt{ICCAD16-2} & 16             & 56    & 967    & 7         \\
		\texttt{ICCAD16-3} & 16             & 1100  & 3916   & 7         \\
		\texttt{ICCAD16-4} & 16             & 157   & 1678   & 7         \\ \hline
	\end{tabular}
\end{table}

\Cref{tab:bench} lists the benchmark details.
To verify the efficiency of our proposed method on EUV oriented designs, we shrink \texttt{ICCAD16} layouts to reach a CD under 7$nm$ technology node as indicated in the column ``CD ($nm$)''.
Columns ``HS \#'' and ``NHS \#'' are numbers of hotspot and non-hotspot clips in each benchmark
and ``Tech ($nm$)'' is the technology nodes of each design.
\texttt{ICCAD12} contains all the 28$nm$ clips with labels which can be directly input to our flow. Total count of hotspot clips and non-hotspot clips are 3728 and 159672, respectively.
\texttt{ICCAD16} contains four layouts that are original designed for fuzzy matching tasks.
To locate defects in those layouts, we apply industrial optical proximity correction (OPC) and layout manufacturability checker (LMC) tools on scaled layouts using EUV lithography models for 7$nm$ metal layer.
In the LMC stage, we only consider three types of defects that are edge placement error, bridge and neck which contribute most to circuit failures.
Then all the locations where edge placement error, bridging and necking occur are marked as defects.
To perform efficient and parallel testing, clip-based scan is usually applied in classic hotspot detection flow,
where the clip size and scanning stride are empirically determined according to the $D$ (by \Cref{eq:isodis}) under given lithography specifications.
Figure~\ref{fig:odvsacc} shows that fail detected hotspot count of exact pattern matching reduces to zero as clip size increases to around $3\times D=690nm$ that will be chosen as the clip size in our experiment.

\begin{figure}[tb!]
	\centering
	\subfloat[Influence of clip size]{\pgfplotsset{
    width =0.14\textwidth,
    height=0.114\textwidth,
    xmin=0.3, xmax=3.1
}
\begin{tikzpicture}[scale=1]
\begin{axis}[minor tick num=0,
scale only axis,
ymax=100.0,
ymin=90,
yticklabel style={/pgf/number format/.cd, fixed, fixed zerofill, precision=2, /tikz/.cd},
y label style={at={(axis description cs:-0.35,0.5)},rotate=0,anchor=south},
xlabel={Clip Size ($\times D$)},
ylabel={Accuracy (\%)},
xlabel near ticks,
legend style={
  draw=none,
  at={(0.50,1.2)},
  anchor=north,
  legend columns=3,
}
]

\addplot +[line width=0.7pt] [color=item2p,   solid, mark=none]  table [x={ratio},  y expr=\thisrow{acc}*100]  {svm.dat};
\addplot +[line width=0.7pt] [color=item1,   dashed, mark=none] coordinates {(0.9, 0) (0.9, 1)};
\node[label={{[xshift=0cm, yshift=-0.4cm](200$nm$,98.9\%)}}] (s) at (axis cs:1.5,95) {};
\node (d) at (axis cs:0.9,98.9) {};
\draw[->] (s)--(d);


\end{axis}

\end{tikzpicture}\label{fig:odvsacc}}
	\subfloat[Clip-based scan]{\includegraphics[width=0.25\textwidth]{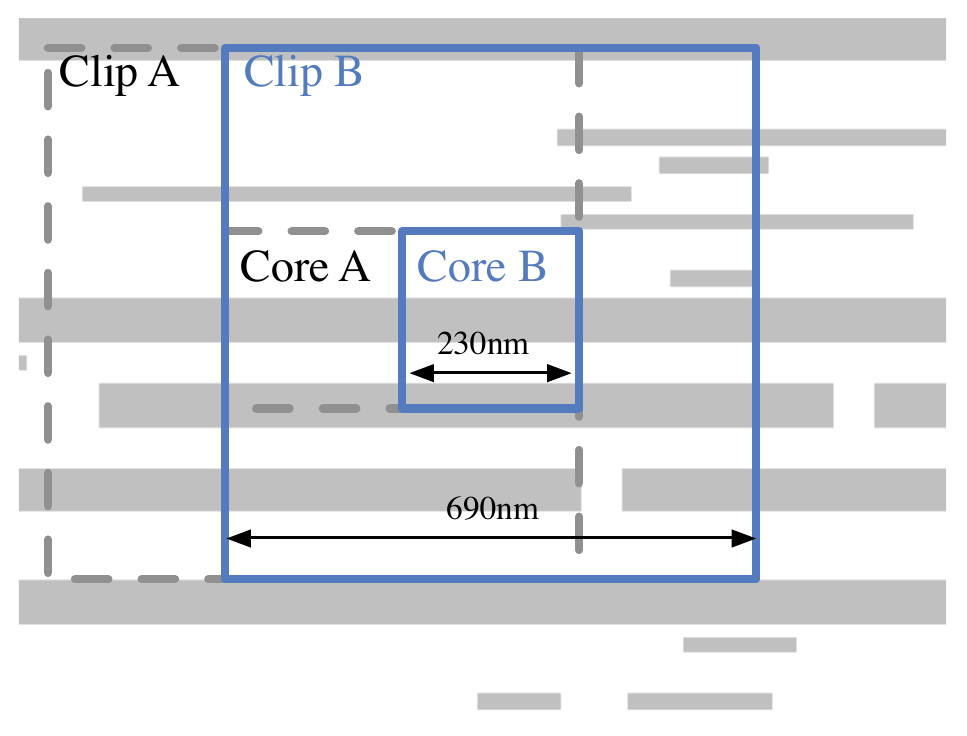}\label{fig:dispatch}}
	\caption{Dispatching layouts based on estimated $D$. Because there is no spacing and overlapping between adjacent core regions of adjacent clips, each layout is fully scanned in the sampling and detection flow.
	Particularly, exact matching has a detection rate of 98.9\% with the clip size in the original contest setting \cite{HSD_ICCAD2016_Rasit}.}
\end{figure}

According to the estimated $D$ of 7$nm$ EUV lithography system, 
we adopt an overlapped dispatching method that covers the whole layout with reasonably small clip size that contains enough information to determine whether the center core region is hotspot or not.
\Cref{fig:dispatch} illustrates the details of the dispatching procedure.
We use a $690\times 690$ sliding window to scan the whole layout with scanning stride being $\frac{1}{3}$ of the clip size,
which ensures that center $230\times 230$ core regions of each clips are exactly covering the whole chip.
Note that to ensure that clips contain more than 96\% information to estimate the printability of their core region,
the smallest distance from the core boundary to the clip boundary is intentionally selected as 230$nm$.
Furthermore, each clip will be marked as hotspot clip if defects occur at its core region as shown in \Cref{fig:dispatch}.
Statistics of \texttt{ICCAD16} benchmarks are also listed in columns ``HS \#'' and ``NHS \#''.
We can notice that the smallest layout \texttt{ICCAD16-1} is defect-free, therefore the case \texttt{ICCAD16-1} is ignored in following experiments.
For the rest of \texttt{ICCAD16} cases, \texttt{ICCAD16-2} has 56 hotspots out of 1023 clips, \texttt{ICCAD16-3} has 1100 hotspots out of 5016 clips and \texttt{ICCAD16-4} has 157 hotspots out of 1835 clips.
It should be noted that although layout \texttt{ICCAD16-4} is much larger than other cases, it is much more regular and a large fraction of the patterns are clearly EUV friendly,
we therefore only extract clips from more problematic regions and that is why the total clip count is less than  \texttt{ICCAD16-3}.
We can also see the out of expected behaviors on \texttt{ICCAD16-4} in the experiments in the following sections.

To accommodate the shallow neural networks and the computational requirements, we conduct feature tensor extraction on each clips in all benchmark cases. 
The settings of \texttt{ICCAD12} and \texttt{ICCAD16} are listed in \Cref{tab:fte},
where we use the same settings as \cite{HSD_DAC2017_Yang} for case \texttt{ICCAD12}.
For the case \texttt{ICCAD16}, we pick a grid size that is close to the design pitch (i.e.~32$nm$).

\begin{table}[tb]
	\centering
	\caption{Feature Tensor Settings}
	\label{tab:fte}
	\renewcommand{\arraystretch}{1.2}
	\setlength{\tabcolsep}{8pt}
	\begin{tabular}{|c|ccc|}
		\hline
		Benchmarks & Grid  & Grid Size & Clip Size ($nm^2$)\\ \hline \hline
		\texttt{ICCAD12}  & 12$\times$12 & 100$\times$100 & 1200$\times$1200 \\
		\texttt{ICCAD16}  & 23$\times$23 & 30$\times$30  & 690$\times$690  \\ \hline
	\end{tabular}
\end{table}

\subsection{Effectiveness of Batch Active Sampling}

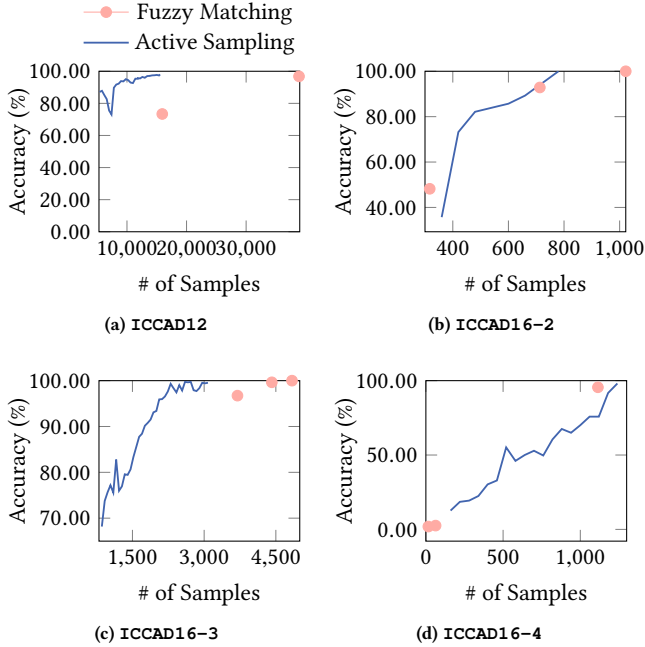
\begin{figure}[tb!]
	\centering
	\subfloat[\texttt{ICCAD12}]  {\pgfplotsset{
	width =0.15\textwidth,
	height=0.12\textwidth,
	xmin=5300, xmax=39000,
	scaled x ticks=false
}
\begin{tikzpicture}[scale=1]
\begin{axis}[minor tick num=0,
scale only axis,
ymin=0.0,
ymax=100.0,
xtick={10000,20000,30000},
yticklabel style={/pgf/number format/.cd, fixed, fixed zerofill, precision=2, /tikz/.cd},
y label style={at={(axis description cs:-0.3,0.5)},rotate=0,anchor=south},
xlabel={\# of Samples},
ylabel={Accuracy (\%)},
xlabel near ticks,
legend style={
	draw=none,
	at={(0.46,1.5)},
	anchor=north,
	legend columns=1,
}
]

\addplot[mark=*, color=item1] coordinates {(15923,73.38)};

\addplot +[line width=0.7pt] [color=item2p,   solid, mark=none]  table [x={alpha},  y expr=\thisrow{acc}*100]  {svm.dat};
\addplot[mark=*, color=item1] coordinates {(38879,96.83)};
\legend{Fuzzy Matching, Active Sampling}

\end{axis}

\end{tikzpicture} \label{fig:acc12}}
	\subfloat[\texttt{ICCAD16-2}]{\pgfplotsset{
    width =0.15\textwidth,
    height=0.12\textwidth,
    xmin=300, xmax=1022
}
\begin{tikzpicture}[scale=1]
\begin{axis}[minor tick num=0,
scale only axis,
ymax=100,
yticklabel style={/pgf/number format/.cd, fixed, fixed zerofill, precision=2, /tikz/.cd},
y label style={at={(axis description cs:-0.3,0.5)},rotate=0,anchor=south},
xlabel={\# of Samples},
ylabel={Accuracy (\%)},
xlabel near ticks,
legend style={
  draw=none,
  at={(0.50,1.2)},
  anchor=north,
  legend columns=3,
}
]
\addplot[mark=*, color=item1] coordinates {(317,48.21)};
\addplot[mark=*, color=item1] coordinates {(713,92.86)};
\addplot[mark=*, color=item1] coordinates {(1022,100)};
\addplot +[line width=0.7pt] [color=item2p,   solid, mark=none]  table [x={alpha},  y expr=\thisrow{acc}*100]  {svm.dat};


\end{axis}

\end{tikzpicture}  \label{fig:acc2}}\\
	\subfloat[\texttt{ICCAD16-3}]{\pgfplotsset{
	width =0.15\textwidth,
	height=0.12\textwidth,
	xmin=800, xmax=5000
}
\begin{tikzpicture}[scale=1]
\begin{axis}[minor tick num=0,
scale only axis,
ymax=100,
xtick={1500,3000,4500},
yticklabel style={/pgf/number format/.cd, fixed, fixed zerofill, precision=2, /tikz/.cd},
y label style={at={(axis description cs:-0.3,0.5)},rotate=0,anchor=south},
xlabel={\# of Samples},
ylabel={Accuracy (\%)},
xlabel near ticks,
legend style={
	draw=none,
	at={(0.50,1.2)},
	anchor=north,
	legend columns=3,
}
]
\addplot[mark=*, color=item1] coordinates {(3697,96.73)};
\addplot[mark=*, color=item1] coordinates {(4416,99.64)};
\addplot[mark=*, color=item1] coordinates {(4838,100)};

\addplot +[line width=0.7pt] [color=item2p,   solid, mark=none]  table [x={alpha},  y expr=\thisrow{acc}*100]  {svm.dat};


\end{axis}
\end{tikzpicture}  \label{fig:acc3}}
	\subfloat[\texttt{ICCAD16-4}]{\pgfplotsset{
	width =0.15\textwidth,
	height=0.12\textwidth,
	xmin=0, xmax=1300
}
\begin{tikzpicture}[scale=1]
\begin{axis}[minor tick num=0,
scale only axis,
ymax=100,
yticklabel style={/pgf/number format/.cd, fixed, fixed zerofill, precision=2, /tikz/.cd},
y label style={at={(axis description cs:-0.3,0.5)},rotate=0,anchor=south},
xlabel={\# of Samples},
ylabel={Accuracy (\%)},
xlabel near ticks,
legend style={
	draw=none,
	at={(0.50,1.2)},
	anchor=north,
	legend columns=3,
}
]
\addplot[mark=*, color=item1] coordinates {(16,1.91)};
\addplot[mark=*, color=item1] coordinates {(63,2.55)};
\addplot[mark=*, color=item1] coordinates {(1114,95.54)};

\addplot +[line width=0.7pt] [color=item2p,   solid, mark=none]  table [x={alpha},  y expr=\thisrow{acc}*100]  {svm.dat};


\end{axis}
\end{tikzpicture}  \label{fig:acc4}}
	\caption{Learning model performance v.s.~sampling count. 
		The blue curve is the reference performance obtained from fuzzy matching with different area constrains reflected as different sampling count.
		The red curve shows the sampling results based on \Cref{alg:bas}.
		}
	\label{fig:bal}
\end{figure}

In the first experiment, we will compare the batch active sampling method with fuzzy matching under different area constraints.
The procedures of \Cref{alg:bas} on four benchmark sets are depicted in \Cref{fig:bal},
where the x-axis represents the total number of patterns sampled into training set and the y-axis denotes the detection accuracy.
According to the analysis in \Cref{sec:analysis}, we avoid choosing the $\frac{k}{n}$ that results in big rounding error (i.e.~0.5).
Considering that sample count also affects the training performance and the lithography simulation overhead, we pick $k=60,n=90$ in all the \texttt{ICCAD16} benchmarks.
On the other hand, \texttt{ICCAD12} is a much larger benchmark set that contains more than 150,000 clips, therefore a smaller $\frac{k}{n}$ (i.e.~0.05) is chosen to limit the total number of sampled clips.

The discrete dots in \Cref{fig:bal} correspond to fuzzy matching results with area constraints 90\%, 95\% and 100\%, respectively. 
It can be seen that our batch sampling converges at a reasonably high detection accuracy on both DUV and EUV specific layers
while requiring much less training instances than exact pattern matching.
In other words, our proposed method can significantly reduce lithography simulation overhead.
Particularly for the case \texttt{ICCAD12}, exact matching samples more than $10^5$ clips among the whole data set, while our method achieves similar results with only 5799 clips.

\begin{table*}[tb!]
	\centering
	\caption{Full chip pattern sampling and hotspot detection on ICCAD12/16 benchmarks.}
	\label{tab:rst}
	\renewcommand{\arraystretch}{1.2}
	\setlength{\tabcolsep}{1.6pt}
\begin{tabular}{|c|cc|cc|cc|cc|cc|cc|cc|}
	\hline
	\multirow{2}{*}{Benchmarks} & \multicolumn{2}{c|}{PM\_exact} & \multicolumn{2}{c|}{PM\_a95} & \multicolumn{2}{c|}{PM\_a90} & \multicolumn{2}{c|}{PM\_e2} & \multicolumn{2}{c|}{FT} & \multicolumn{2}{c|}{Greedy} & \multicolumn{2}{c|}{Ours} \\
	                    & Acc (\%)   & Litho     & Acc (\%)      & Litho         & Acc (\%)       & Litho         & Acc (\%)       & Litho      & Acc (\%)     & Litho    & Acc (\%)       & Litho      & Acc (\%)      & Litho     \\ \hline \hline
	\texttt{ICCAD12}    & 100.00     & 127746    & 96.83$^\dag$  & 38879$^\dag$  & 73.38$^\dag$   & 15923$^\dag$  & 100.00               &124320        &  32.14  &  20000&24.57     & 26945     & 98.02         & 16719      \\
	\texttt{ICCAD16-2}  & 100.00     & 1022      & 92.86         & 717           & 48.21          & 328           & 100.00         & 1022       &91.07        &782          & 51.79          & 475        & 100.00        & 944       \\
	\texttt{ICCAD16-3}  & 100.00     & 4838      & 99.64         & 4420          & 96.73          & 3717          & 99.91          & 4777      &86.18         &1854          & 73.82          & 2496       & 99.54         & 3824      \\
	\texttt{ICCAD16-4}  & 95.54      & 1134      & 2.55          & 65            & 1.91           & 20            & 78.34          & 842       &50.32         &573          & 50.32          & 668        & 98.09        & 1709       \\ \hline \hline
	Average             & 98.88      & 33685     & 72.97         & 11043         & 55.06          & 4997          & 94.56   &    32488  &  65.88     &  6118  &      50.12  &    7646   &    98.91         & 5799          \\
	Ratio               & 1.000      & 5.809     & 0.74          & 1.904         & 0.557          & 0.862         & 0.956      &   5.602   & 0.666     & 1.055  &    0.507     &  1.319  &  \textbf{1.000}       &     \textbf{1.000} \\ \hline
\end{tabular}
\begin{tablenotes}
    \item 
    \dag Experiments are conducted on the center $600\times600$ region of each clip  because the area constrained fuzzy matching cannot be finished within one week using original clip size of $1200\times1200$.
\end{tablenotes}
\end{table*}
\subsection{Comparison with Existing Methods}
We compare the sampling results on \texttt{ICCAD12/16} with exact/fuzzy matching methods and two recent sampling methods, as listed in \Cref{tab:rst}.
Columns ``PM\_exact'', ``PM\_a95'', ``PM\_a90'', ``PM\_e2'' correspond to the results derived from pattern matching using a state-of-the-art pattern analysis tool \cite{HSD_DAC2017_Chen},
where ``PM\_exact'' denotes only exactly same patterns can be clustered together,
``PM\_a95'' and ``PM\_a90'' refer to any clips that satisfy 95\% and 90\% area constraints are clustered together
and ``PM\_e2'' groups clips with less than 2$nm$ edge displacements.
Here the area and edge constraints are defined following \cite{HSD_ICCAD2016_Rasit}.
Column ``FT'' lists the result of clustering on frequency domain of layout patterns that is similar to the flow proposed in \cite{HSD_SPIE2014_Shim}.
Column ``Litho'' denotes the number of clips being labeled according to the lithography simulation results, including the clips sampled into training sets and all the detection extras.

``PM\_exact'', as the reference method, shows 100\% accuracy on \texttt{ICCAD12}, \texttt{ICCAD16-2} and \texttt{ICCAD16-3}.
According to the lithography simulation results of the layout in \texttt{ICCAD16-4}, 
we notice all defects appear at the patterns belong to a different design space,
which possibly makes the lithography model and optical proximity analysis inaccurate.
Therefore, we observe minor prediction error and extra on \texttt{ICCAD16-4}.
The result also shows exact pattern matching can achieve extremely high verification accuracy, however, at the cost of simulating and labeling more than 95\% clip patterns in the whole dataset. 
On the contrary, the proposed batch sampling method achieves almost the same detection accuracy querying only 17\% of total layout clips.
It should be noted that it is normal that the instances in a training set is not completely separable, which explains that our method behaves even better than exact pattern matching.
For three fuzzy matching options, varies area or edge constraints offers different level of trade-offs between verification performance and lithography overhead.
``PM\_a95'' and ``PM\_e2'' can still maintain good prediction accuracy on \texttt{ICCAD16-2} and \texttt{ICCAD16-3} with slightly less litho count,
but the total number of labeled instances is still much larger than our method.
Moreover, fuzzy matching fails to extract problematic instances on a more difficult testcase \texttt{ICCAD16-4} with looser constraints
that they all reach less than 50\% prediction accuracy.

\cite{HSD_DAC2013_Zhang,HSD_JM3_2015_Shim} propose to use the frequency domain representation to sample layout patterns with similar property and detect hotspots.
Here we conduct additional experiments by clustering layout clips based on their Fourier Transform results.
Clips closest to a cluster center will be selected as the representative clip that indicates the property of the whole cluster.
By the results in the column ``FT'', we can observe that with similar sampling number, batch active sampling exhibits much better than frequency domain clustering.

We also conduct an experiment using greedy sampling method \cite{HSD_ICCAD2016_Zhang} where each instance being predicted as hotspot will be incrementally added into the training set.
Although greedy sampling method can successfully select partial hotspot clips in some test cases, the performance is highly affected by the initial learning model.
As listed in ``Greedy'', the greedy method \cite{HSD_ICCAD2016_Zhang} only achieves 50.12\% detection accuracy on average.

\section{Conclusion}
\label{sec:conclu}
A layout pattern sampling and hotspot detection flow is proposed to adaptively sample layout patterns into a pattern library that is used to train a machine learning model for layout hotspot detection.
The diversity-aware batch sampling and the interactive optimization of learning model can efficiently select interesting patterns and ensure a better model generality.
Experiments show that the proposed framework is able to achieve similar detection accuracy requiring less than 20\% of labeled patterns,
which reduces lithography simulation overhead by a significant amount.
Continuing study on model robustness, feature extraction and training set initialization are also interesting to fit the proposed framework better on modern IC design requirements.

{
\normalsize
\bibliographystyle{IEEEtran}
\bibliography{./ref/Top-sim,./ref/DFM,./ref/HSD,./ref/Software,./ref/LEARN,./ref/DL,./ref/MTRX}
}


\end{document}